\documentclass{amsart}
\usepackage[marginratio=1:1]{geometry}

\usepackage[hidelinks]{hyperref}
\usepackage{calc}
\newsavebox\CBox
\newcommand\hcancel[2][0.5pt]{%
  \ifmmode\sbox\CBox{$#2$}\else\sbox\CBox{#2}\fi%
  \makebox[0pt][l]{\usebox\CBox}%
  \rule[0.5\ht\CBox-#1/2]{\wd\CBox}{#1}}
\usepackage{blindtext}
\usepackage{color}
\usepackage{hyperref,url}
\usepackage{float}
\usepackage{hyperref}
\usepackage{enumerate}
\usepackage{enumitem}   
\usepackage{bbm}
\usepackage{cancel}
\usepackage{mathrsfs}
\usepackage{colonequals}
\usepackage{pdfpages}

\usepackage{amsmath,amsfonts,amsthm,bm}
\usepackage{mathrsfs}  
\usepackage{xcolor}
\usepackage{amsfonts}
\usepackage{amssymb}
\usepackage{amsmath}
\usepackage{mathtools}
\usepackage{mathrsfs}  

\usepackage[normalem]{ulem}

\numberwithin{equation}{section}
\usepackage[toc,page]{appendix}
\usepackage{tikz-cd}
\theoremstyle{definition}

\newtheorem{definition}{Definition}[section]
\newtheorem{remark}[definition]{Remark}

\theoremstyle{plain}

\newtheorem{theorem}[definition]{Theorem}

\newtheorem{lemma}[definition]{Lemma}

\usepackage{hyperref}

\newcommand{\R}{\mathbb R}
\newcommand{\RR}{\mathbb R}

\newcommand{\ReLU}{\mathrm{ReLU}}

\usepackage{scalerel}

\definecolor{roxo}{rgb}{0.44, 0.16, 0.39}
\definecolor{ao(english)}{rgb}{0.0, 0.5, 0.0}
\definecolor{dmagenta}{RGB}{139, 0, 139}
\definecolor{dgreen}{RGB}{0,90,0}
\definecolor{navy}{RGB}{0,0,128}

\usepackage{stackengine}

\definecolor{iblue}{RGB}{0, 35, 194}
\title[Optimal Approximation Complexity of  Neural Networks]{Optimal Approximation Complexity of High-Dimensional Functions with Neural Networks}
\author[V.P.H. Goverse, J. Hamdan, J. Tanner]
{Vincent P.H. Goverse$^{*}$, Jad Hamdan$^{\dagger}$, and Jared Tanner$^{\dagger}$
}

\address{Department of Mathematics, Imperial College London \& Mathematical Institute, University of Oxford, United Kingdom}
\email{ vincent.goverse21@imperial.ac.uk, hamdan@maths.ox.ac.uk, tanner@maths.ox.ac.uk}

\allowdisplaybreaks
\begin{document}



\maketitle
\vspace{-1.3cm}
\begin{align*}
   &\small {}^* \textit{Imperial College London}\\
   &\small {}^\dagger \textit{University of Oxford}
\end{align*}
\begin{abstract}

We investigate properties of neural networks that use both ReLU and $x^2$ as activation functions and build upon previous results to show that both analytic functions and functions in Sobolev spaces can be approximated by such networks of constant depth to arbitrary accuracy, demonstrating optimal order approximation rates across all nonlinear approximators, including standard ReLU networks. We then show how to leverage low local dimensionality in some contexts to overcome the curse of dimensionality, obtaining approximation rates that are optimal for unknown lower-dimensional subspaces.
\end{abstract}

\section{Introduction}
\label{submission}

The number of parameters needed to approximate smooth high-dimensional functions, $W^{n,\infty}([0,1]^d)$, within a prescribed $\epsilon$ accuracy in the $\ell_{\infty}$  norm was lower bounded by \cite{devore} to have a dependence on $\epsilon$ that is proportional to $\epsilon^{-d/n}$. \cite{yarotsky} has subsequently shown that a simple feedforward neural network with $\ReLU(x):=\max\{0,x\}$ nonlinear activation is nearly optimal in terms of the number of parameters needed, requiring only $c(n,d)=\epsilon^{-d/n}\log(1/\epsilon)$ parameters\footnote{The function $c(n,d)$ depends on the smoothness, $n$, and the dimension of $f(x)$, but not on the desired accuracy $\epsilon$.}, see \cite{yarotsky}[Theorem 1]. Subsequently, \cite{boulle} reduced the number of parameters needed by a feedforward neural network to achieve $\epsilon$ accuracy to being proportional to $c(n,d)=\epsilon^{-d/n}\log(\log(1/\epsilon))$ by using trainable rational function as nonlinear activations, see \cite{boulle}[Theorem 4].  

Here we further adapt the proof by Yarotsky to achieve the optimal dependence of $\epsilon^{-d/n}$ proven by \cite{devore}, using a feedforward network that makes use of {\em two} nonlinear activations (henceforth referred to as {\em bi-activation} networks).  Specifically, we allow some layers to use the ReLU nonlinear activation to localize $f(x)$ through a partition of unity, and the quadratic activation $x^2$ to allow for efficient computation of localized high degree polynomial approximations.

Specifically, following the notation of \cite{devore} and \cite{yarotsky}, we consider nonlinear approximation methods $M_p(a)$ that have a continuous dependence\footnote{The continuous dependence of $M_n(a)$ on $a$ is introduced in \cite{devore} to avoid space filling curves and can be viewed as ensuring the parameters $a$ can be learned from a sufficiently near estimate; for details see \cite{devore}.} on the $p$ parameters $a\in\R^p$ and which approximate high dimensional functions $f(\cdot)$ within the unit ball of the Sobolev space $W^{n,\infty}([0,1]^d)$, 
\begin{equation}\label{sobolev}
    ||f||_{W^{n,\infty}([0,1]^d)}=\text{max }_{\textbf{n}\in|\textbf{n}|\leq n}\text{esssup}_{\textbf{x}\in[0,1]^d} |D^{\textbf{n}}f(\textbf{x})|
\end{equation}
where $\textbf{n}=(n_1,...,n_d)\in \{0,1,...\}^d$, $|\textbf{n}|=\sum_i^d n_i$ and $D^{\textbf{n}}f$ the respective weak derivative.  The foundational lower bound on the number of elements in {\em any} nonlinear approximation method $M_p(a)$ that depends smoothly on $a\in\RR^p$ is given in Theorem \ref{thm:devore}.
\begin{theorem}[Optimal non-linear approximation lower bound, \cite{devore}]\label{thm:devore}
For function $f(x)$ with $||f||_{W^{n,\infty}([0,1]^d)}\le 1$, and $M_p(a)$ depending continuously on $a\in\RR^p$, approximating $f(x)$ with bound
\[
\inf_{M_p(a),a}\max_{x\in [0,1]^d} |f(x)-M_p(a)(x)|\le \epsilon
\]
then necessarily $M_p(\cdot)$ has $p\ge C_1(d,n) \epsilon^{-d/n}$ where $C_1(d,n)$ may depend on $d$ and $n$, but not on $\epsilon$.
\end{theorem}

As a method to explain the value of depth in deep learning, \cite{yarotsky} constructed a feed forward networks with is near optimal order number of parameters as a function of approximation accuracy $\epsilon$.  In particular, 
\begin{theorem}[Near optimal non-linear approximation with ReLU-networks, \cite{yarotsky}]\label{thm:yarotsky}
For function $f(x)$ with $||f||_{W^{n,\infty}([0,1]^d)}\le 1$, there exists $M_{\mathcal C_Y,\ReLU}(a)$ formed as a feed-forward network with at most $\mathcal C_Y=C_2(d,n)\epsilon^{-d/n}(1+\log(1/\epsilon))$ elements $a\in\RR^{\mathcal C_Y}$ for which 
\[
\min_{a\in\RR^{\mathcal C_Y}}\max_{x\in [0,1]^d} |f(x)-M_{\mathcal{C}_Y,ReLU}(a)(x)|\le \epsilon
\]
where $C_2(d,n)$ may depend on $d$ and $n$, but not on $\epsilon$.
\end{theorem}
The feed-forward network $M_{\mathcal C_Y,ReLU}$ constructed in \cite{yarotsky} has hidden layers $h_{i+1}=\ReLU(W_i h_i+b_i)$ for $i=0,\ldots L$ with input $h_0:=x$, $W_i$ being matrices of width bounded independent of $\epsilon$, and depth $L\le c(d,n)(\log(1/\epsilon)+1)$.  The feed-forward network is constructed analogously to the proof in \cite{devore} where there the input $x\in\RR^d$ is partitioned into exponentially many localized portions, each of which then has a local polynomial constructed to approximate $f(\cdot)$.  The ReLU nonlinear activation allows for partitions of the input space $[0,1]^d$ and the logarithmic depth is needed to construct high-degree local polynomial approximations using the saw-tooth functions developed by Telgarsky \cite{telgarsky}; for details, see \cite{yarotsky}.

Our main contribution here is a feed-forward network $M_{\mathcal C_F,bi-\sigma}$ where the layers $h_{i+1}=\sigma_i(W_i h_i+b_i)$ have non-linear activations $\sigma_i(x)$ which are either $\ReLU(x)$ or $x^2$ depending on the layer.  This choice of nonlinear activations is made to simplify the proof in \cite{yarotsky} by retaining the ability to localize $\RR^d$ while more efficiently computing higher-order polynomial functions with bounded depth $L$. Other choices of localizing and approximation activations are possible, see the details of the proof of Theorem 1.3.
\begin{theorem}[Optimal approximation order bi-activation networks]\label{thm:main}
For function $f(x)$ with $||f||_{W^{n,\infty}([0,1]^d)}\le 1$, there exists $M_{\mathcal C_F,bi-\sigma}(a)$ formed as a feed-forward network with $\mathcal C_F=C_3(d,n)\epsilon^{-d/n}$ elements $a\in\RR^{\mathcal{C}_F}$ for which 
\[
\min_{a\in\RR^{\mathcal C_F}}\max_{x\in [0,1]^d} |f(x)-M_{\mathcal C_F,bi-\sigma}(a)(x)|\le \epsilon
\]
where $C_3(d,n)$ may depend on $d$ and $n$, but not on $\epsilon$.
\end{theorem}
The Proof of Theorem \ref{thm:main} is given in Section \ref{sec:proof}, making use of a key lemma from the proof of Theorem \ref{thm:yarotsky} by Yarotsky.

We further extend Theorem \ref{thm:main} in two separate directions, by considering $f(x)$ to be analytic or $f(x)$ to be contained on the union of $d_\text{eff}<d$ dimensional canonical subspaces of $\RR^d$.

\begin{theorem}[Optimal approximation order bi-activation networks: Analytic functions]\label{thm:analytic}
Let $f(x)$ be an analytic function on $[0,1]^d$, characterised \cite{ahlfors} by 
\begin{equation}\label{eq:analfunc}
\sup_{x \in [0,1]^d} \left | \frac{\partial^{\mathbf n} f}{\partial x^{\mathbf n}}(x) \right | \leq C_f^{|{\mathbf n}|+1}{\mathbf n}! \quad\mbox{for}\;\;\mbox{all}\; n\end{equation}
where $C_f$ depends on the particular choice of $f(x)$.
Then for any $d$, and $\epsilon \in (0,1)$, there exists $M_{\mathcal C_A,bi-\sigma}(a)$ formed as a feed-forward network with $\mathcal C_A=C_4(d,C_f)\left((2\epsilon)^{\log^{-\frac{1}{2}}\left(\tfrac{2^d}{\epsilon}\right)}\log^{\frac{d}{2}}\left(\frac{1}{\epsilon}\right)\right)$ elements $a\in\RR^{\mathcal C_A}$ for which 
\[
\min_{a\in\RR^{\mathcal{C_A}}}\max_{x\in [0,1]^d} |f(x)-M_{\mathcal C_A,bi-\sigma}(a)(x)|\le \epsilon
\]
where $C_4(d,C_f)$ does not depend on $\epsilon$.
\end{theorem}
Theorem \ref{thm:analytic} differs from Theorem \ref{thm:main} primarily in the lack of dependence on smoothness $n$ as the number of parameters $\mathcal{C_A}$ needed in the network has been minimized over all admissible $n$.  The consequence of choosing the optimal smoothness $n$ is that the $\epsilon$ and $d$ dependence of the number of parameters $\mathcal{C_A}$ decreases from $(\epsilon^{-1/n})^d$ to predominantly $\log(1/\epsilon)^{d/2}$.

Next, for $d_\text{eff} < d$ we define the canonical subspace of $[0,1]^d$ of dimension $d_\text{eff}$; that is  $$x\in\chi^d_{d_\text{eff},{e}}:=\{x\in[0,1]^d: \mbox{ with if } i \notin {e}, x_i = 0\}.$$ Where $e$ is a subset of $\{1,\dots,d\}$, with $d_\text{eff}$ elements. $I^d_{d_\text{eff}}$ is the collections of all $e$. Then if $f(x)$ is nonzero on only one known subspace $\chi^d_{d_\text{eff},e}$ Lemma \ref{lem:ssubspaces} holds.  In the case that $f(x)$ is nonzero on the union of all $d \choose d_\text{eff}$ such subspaces $$\bar{\chi}^d_{d_\text{eff}} := \bigcup_{e\in I}\chi^d_{d_\text{eff},{e}},$$ the number of parameters $\mathcal C_M$ needed to compute an $\epsilon$ approximation of $f(x)$ over one or all canonical subspaces is given by $\mathcal C_M=C_5(d, d_\text{eff},n)\epsilon^{-d_\text{eff}/n})$ (see Lemma \ref{lem:ssubspaces} and Theorem \ref{thm:subspaces}).

\begin{theorem}[Optimal approximation order bi-activation networks: low-dimensional subspaces]\label{thm:subspaces}
For function $f(x)$ with $||f||_{W^{n,\infty}([0,1]^d)}\le 1$ where $x$ is restricted to $\bar{\chi}^d_{d_\text{eff}}$, there exists $M_{\mathcal C_M,bi-\sigma}(a)$ formed as a feed-forward network with $\mathcal C_M=C_5(d,n)\epsilon^{-d_\text{eff}/n}$ elements $a\in\RR^{\mathcal C_M}$ for which the error restricted on $\chi_{d_\text{eff}}^d$ is
\[
\min_{a\in\RR^{\mathcal C_M}}\max_{x\in \bar{\chi}_{d_\text{eff}}^d} |f(x)-M_{\mathcal C_M,bi-\sigma}(a)(x)|\le \epsilon
\]
where $C_5(d,n)$ may depend on $d$ and $n$, but not on $\epsilon$.
\end{theorem}
This restricted subspace model is motivated by natural image inputs with prescribed compression on a known orthogonal basis, such as JPEG compression.  This union of subspace model $\bar{\chi}_{d_eff}^d$ is also widely used in the theory of compressed sensing, see \cite{foucart13} and references therein, and has also been used to increase robustness against adversarial attacks on image classification by \cite{guo2018countering}.

\section{Approximation power of bi-activation networks}
The proof of Theorem \ref{thm:main} being adapted from that \ref{thm:yarotsky} in \cite{yarotsky}, an understanding of the former is essential in order to explain the latter. 

As mentioned previously, \cite{yarotsky} first partitions the input $x\in [0,1]^d$ into exponentially many localized portions using a partition of unity $\{\phi_{\textbf{m}}\}$, where each $\phi_{\textbf{m}}$ is piecewise linear and expressible by a ReLU network with a constant number of parameters (see Proposition 1 in \cite{yarotsky}). The aim is then to approximate the function $f$ by Taylor polynomials locally, giving the following representation for an approximation of $f$.

\begin{lemma}[\cite{yarotsky}] \label{approximationlemma}
  Let $\epsilon>0$ be arbitrary and $f \in W^{n,\infty}([0,1]^d)$. Then there exists a function $\Tilde{f}$ expressible as
    \[
        \Tilde{f}(x) = \sum_{{\textbf{m}}\in \{0,...,N\}^d} \sum_{{\textbf{n}}: |{\textbf{n}}| < n} a_{{\textbf{m}},\textbf{n}}\phi_{\textbf{m}}(x)\left(x-\frac{{\textbf{m}}}{N}\right)^{\textbf{n}},
    \]
    where $a_{m,n} \in \mathbb{R},\, |a_{m,n}|\leq 1$, $\{\phi_{\textbf{m}}\}_{\textbf{m}\in\{0,1,...,N\}^d}$ is a partition of unity such that each $\phi_\textbf{m}$ is given by a product of $d$ piecewise linear univariate factors. Furthermore, $\Tilde{f}$ is such that
\begin{align}
	|f(x)-\Tilde{f}(x)|
	&\leq \frac{2^dd^n}{n!}\left(\frac{1}{N}\right)^n  \max_{{\textbf{n}}:|{\textbf{n}}|=n}\text{ess sup}_{x\in [0,1]^d} |D^{\textbf{n}}f(x)|. \label{eq:bound}
\end{align}
\end{lemma}
\textit{The proof of this lemma is included in the appendix for completeness.}

Showing that ReLU networks can approximate monomials (and, in turn, polynomials) would then complete the proof. Indeed, in Section 3.1 of \cite{yarotsky}, the author does so by first showing that $f(x)=x^2$ can be approximated by a ReLU network of complexity $O(\ln(1/\epsilon))$. Using the following identity to recover multiplication from squaring:
\begin{equation}\label{interpolation}
    xy = \frac{1}{2}\big((x+y)^2-x^2-y^2\big)
\end{equation}
the author then shows how a ReLU network of complexity $O(\ln(1/\epsilon))$ can in fact approximate terms of the form $\phi_{\textbf{m}}(x)\left(x-\frac{{\textbf{m}}}{N}\right)^{\textbf{n}}$.

Lastly, note that in lemma \ref{approximationlemma}, $\tilde{f}$ is a linear combination of at most $d^n(N+1)^d$ such terms. $N$ is a smoothness parameter that can be chosen so that the upper bound in \eqref{eq:bound} becomes $|f(x)-\Tilde{f}(x)|<\epsilon$. In Yarotsky's case, this corresponds to choosing
\begin{equation}\label{N}
	N = N(\epsilon, d, n) =\left\lceil\left(\frac{n!}{2^d d^n}\epsilon\right)^{-1/n}\right\rceil,
\end{equation}
which also yields 
\[
d^n(N+1)^d = d^n\left(\frac{n!}{2^dd^n}\epsilon\right)^{-d/n}=O(\epsilon^{-d/n}),
\]
and the final ReLU network used approximate $f$ therefore consists of $\mathcal{C}_Y=O(\epsilon^{-d/n}\ln(1/\epsilon))$ parameters due to the $\log(1/\epsilon)$ depth needed to approximate $x^2$ within $\epsilon$ using a ReLU network.


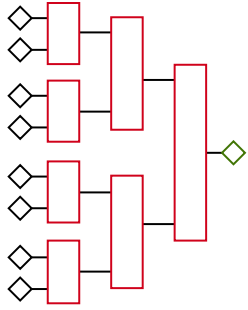
\begin{figure}
    \centering

\tikzset{every picture/.style={line width=0.75pt}} 

\tikzset{every picture/.style={line width=0.75pt}} 

\begin{tikzpicture}[x=0.75pt,y=0.75pt,yscale=-1,xscale=1]

\draw    (150.86,136.63) -- (158.86,136.63) ;
\draw    (150.86,152.63) -- (158.86,152.63) ;
\draw    (150.86,175.81) -- (158.86,175.81) ;
\draw    (150.86,191.8) -- (158.86,191.8) ;
\draw    (150.86,216.59) -- (158.86,216.59) ;
\draw    (150.86,232.58) -- (158.86,232.58) ;
\draw    (150.86,257.37) -- (158.86,257.37) ;
\draw    (150.86,273.36) -- (158.86,273.36) ;
\draw    (174.86,143.83) -- (190.86,143.83) ;
\draw    (174.86,183.81) -- (190.86,183.81) ;
\draw    (174.86,224.59) -- (190.86,224.59) ;
\draw    (174.86,264.56) -- (190.86,264.56) ;
\draw    (206.86,167.82) -- (222.86,167.82) ;
\draw    (206.86,240.58) -- (222.86,240.58) ;
\draw    (238.86,204.6) -- (246.86,204.6) ;
\draw  [color={rgb, 255:red, 208; green, 2; blue, 27 }  ,draw opacity=1 ] (158.97,128.98) -- (174.86,128.98) -- (174.86,159.82) -- (158.97,159.82) -- cycle ;
\draw  [color={rgb, 255:red, 208; green, 2; blue, 27 }  ,draw opacity=1 ] (158.97,168.16) -- (174.86,168.16) -- (174.86,199) -- (158.97,199) -- cycle ;
\draw  [color={rgb, 255:red, 208; green, 2; blue, 27 }  ,draw opacity=1 ] (158.97,208.94) -- (174.86,208.94) -- (174.86,239.78) -- (158.97,239.78) -- cycle ;
\draw  [color={rgb, 255:red, 208; green, 2; blue, 27 }  ,draw opacity=1 ] (158.97,248.91) -- (174.86,248.91) -- (174.86,280.55) -- (158.97,280.55) -- cycle ;
\draw  [color={rgb, 255:red, 208; green, 2; blue, 27 }  ,draw opacity=1 ] (190.97,136.18) -- (206.86,136.18) -- (206.86,192.95) -- (190.97,192.95) -- cycle ;
\draw  [color={rgb, 255:red, 208; green, 2; blue, 27 }  ,draw opacity=1 ] (190.97,216.13) -- (206.86,216.13) -- (206.86,272.9) -- (190.97,272.9) -- cycle ;
\draw  [color={rgb, 255:red, 208; green, 2; blue, 27 }  ,draw opacity=1 ] (222.97,160.16) -- (238.86,160.16) -- (238.86,248.91) -- (222.97,248.91) -- cycle ;

\draw   (145.07,130.84) -- (150.86,136.63) -- (145.07,142.43) -- (139.28,136.63) -- cycle ;
\draw   (145.07,146.84) -- (150.86,152.63) -- (145.07,158.42) -- (139.28,152.63) -- cycle ;
\draw   (145.07,170.02) -- (150.86,175.81) -- (145.07,181.6) -- (139.28,175.81) -- cycle ;
\draw   (145.07,186.01) -- (150.86,191.8) -- (145.07,197.59) -- (139.28,191.8) -- cycle ;
\draw   (145.07,210.8) -- (150.86,216.59) -- (145.07,222.38) -- (139.28,216.59) -- cycle ;
\draw   (145.07,226.79) -- (150.86,232.58) -- (145.07,238.37) -- (139.28,232.58) -- cycle ;
\draw   (145.07,251.58) -- (150.86,257.37) -- (145.07,263.16) -- (139.28,257.37) -- cycle ;
\draw   (145.07,267.57) -- (150.86,273.36) -- (145.07,279.15) -- (139.28,273.36) -- cycle ;
\draw  [color={rgb, 255:red, 65; green, 117; blue, 5 }  ,draw opacity=1 ] (252.65,198.81) -- (258.44,204.6) -- (252.65,210.39) -- (246.86,204.6) -- cycle ;

\end{tikzpicture}
    \caption{Multiplication of $k$ elements (depicted in black) using $O(k)$ subnetworks (depicted in red) each of constant size (independent of $d$ and $n$), giving a network of depth $O(\ln_2(k))$. Here, $k=8$.}
    \label{fig:1}
\end{figure}

\subsection{Proof of Theorem \ref{thm:main}, Optimal approximation order bi-activation networks}\label{sec:proof}

\begin{proof}[Proof of Theorem \ref{thm:main}]
Let $\Tilde{f}$ be the approximation to $f$ given by Lemma \ref{approximationlemma}. Since $f$ is in the unit-ball in $W^{n,\infty}$, $\max_{{\textbf{n}}:|{\textbf{n}}|=n}\text{ess sup}_{x\in [0,1]^d} |D^{\textbf{n}}f(x)|\leq 1$. Choosing the same $N$ as in \ref{N}, we find that $||f-\Tilde{f}||_\infty\leq \epsilon$.

In contrast to ReLU networks, we claim that bi-activation networks can represent terms of the form $\phi_{\textbf{m}}(x)(x-{\textbf{m}}/N)^{\textbf{n}}$ \textit{exactly} using a constant number of trainable parameters. Indeed, each of these terms is itself a product of at most $d+n-1$ piecewise linear univariate factors: a product of $d$ functions defining each $\phi_{\textbf{m}}$ and at most $n-1$ functions $x_k-m_k/N$. These products can be implemented by a bi-activation network with a complexity of the order of $(n+d)$ and depth of the order of $\log_2(n+d)$ (in both cases, $O(1)$ with respect to $\epsilon$), by repeatedly pairing up the terms and multiplying them in tournament fashion (see figure \ref{fig:1}). The multiplication of two terms can be achieved by a bi-activation network of constant size using \eqref{interpolation}\footnote{More specifically, we can use a network with activation function $x^2$ which has one hidden layer. The inputs $x$ and $y$ connect fully to the hidden layer with three nodes, and weights $[0,1], [1,0]$ and $[1,1]$. The three nodes are connected to the output with weight $[-1/2,-1/2,1/2]$.}. 

Therefore, $\Tilde{f}$ can be written by a bi-activation network $M_{\mathcal C_F,bi-\sigma}(a)$ with $\mathcal{C}_F=O(d^n(N+1)^d)$ parameters as follows. The network uses parallel subnetworks that each compute a term in the series defining $\Tilde{f}$, and computes the final output by summing the outputs of these subnetworks, weighted with the appropriate $a_{{\textbf{m}},\textbf{n}}$. Since there are not more than $d^n(N+1)^d$ subnetworks, $\mathcal{C}_F=C_3(d,n)d^n(N+1)^d$ weights and computation units, for some constant $C_3(d,n)$. For our choice of $N$ in \eqref{N} to achieve an $\epsilon$ accurate approximation, $\mathcal{C}_F=O(\epsilon^{-d/n})$. 
\end{proof}



\subsection{Proof of Theorem \ref{thm:analytic}, Optimal approximation order bi-activation networks: Analytic functions}

\begin{proof}[Proof of Theorem \ref{thm:analytic}]  Once again, let $\Tilde{f}$ be the approximation to $f$ given by Lemma \ref{approximationlemma}, noting that $f\in W^{n,\infty}([0,1]^d)$ for all $n$ as it is analytic. Then applying the bound on $|f(x)-\Tilde{f}(x)|$ given by the same Lemma and the bound on smoothness for analytic functions \eqref{eq:analfunc}, we find that
\begin{align*}
	|f(x)-\Tilde{f}(x)|
	&\leq \frac{2^dd^n}{n!}\left(\frac{1}{N}\right)^n  \max_{{\textbf{n}}:|{\textbf{n}}|=n}\text{ess sup}_{x\in [0,1]^d} |D^{\textbf{n}}f(x)|\\
	& \leq \frac{2^dd^n}{n!}\left(\frac{1}{N}\right)^n C_f^{n+1}n! \\
        & \leq {2^dd^n}\left(\frac{C_f}{N}\right)^{n+1},
\end{align*}
where $C_f$ is a constant depending on $f$.

Notice that in this case the result holds for all $n$. This means that, when picking $N$, we can optimize over $n$ to minimize the number of trainable parameters needed by our network. To begin with, choosing
\begin{equation}\label{eq:N2}
	N_1 = N(C_f, \epsilon, d, n) =\frac{1}{C}\left\lceil\left(\frac{\epsilon}{2^dd^n}\right)^{-1/{(n+1)}}\right\rceil
\end{equation}
we get that $||f-\Tilde{f}||_\infty\leq \epsilon$. 

Arguing in the exact same manner as in the proof of Theorem \ref{thm:main}, we know that $\Tilde{f}$ can be written as a bi-activation neural network $M_{\mathcal C_A,bi-\sigma}(a)$. The total number of parameters $\mathcal C_A$ then needed by the network to represent $\Tilde{f}$ is equal to
\begin{equation}\label{eq:comp}
\mathcal{C}_A = C_4(C_f,n,d)d^n(N+1)^d
\end{equation}
for some constant $C_4=C_4(C_f,n,d)$ that does not depend on $\epsilon$. Substituting the choice of $N$ in \eqref{eq:N2} in \eqref{eq:comp} and minimizing over $n$, we find that $\mathcal{C_A}$ is minimal for
\begin{equation}\label{eq:n1}
n_\text{min}= \sqrt{\frac{ d\left(d \log (2)+\log \left(\frac{1}{\epsilon} \right)\right)}{\log (d)}}.    
\end{equation}

Substituting \eqref{eq:n1} and \eqref{eq:N2} into \eqref{eq:comp}, gives us
\[\mathcal C_A =
C_4\cdot2^{\frac{d^{3/2} \sqrt{\log (d)}}{\sqrt{d \log (2)+\log \left(\frac{1}{\epsilon
   }\right)}}} \epsilon ^{\frac{\sqrt{d} \sqrt{\log (d)}}{\sqrt{d \log (2)+\log
   \left(\frac{1}{\epsilon }\right)}}} \log ^{\frac{d}{2}}\left(\frac{1}{\epsilon
   }\right),
   \]
which grows as $\epsilon \to 0 $ in the order of $$(2\epsilon)^{\log^{-\frac{1}{2}}\left(\tfrac{2^d}{\epsilon}\right)}\log^{\frac{d}{2}}\left(\frac{1}{\epsilon}\right),$$
concluding the proof.

\end{proof}

\subsection{Proof of Theorem \ref{thm:subspaces}, Optimal approximation order bi-activation networks: low-dimensional subspaces}

For clarity, first consider the simplest case of $x\in \chi_{d_\text{eff},e}^d$, for a known $e \in I$.  Without loss of generality this can be the first $d_\text{eff}$ dimensions of $\RR^d$ being nonzero, that is $f \circ A(x) :=f(Ax)$ for  
$A\in \mathbb R^{d \times d_\text{eff}}$ given by 
\begin{equation}\label{eq:Asimple} A = \begin{pmatrix}
1 & 0 & \dots & 0 & 0 \\
0 & 1 & & 0 & 0 \\
\vdots &  & \ddots &  & \vdots \\
0 & 0 &  & 1 & 0 \\
0 & 0 &  & 0 & 1 \\
0 & 0 & \dots & 0 & 0 
\end{pmatrix}.  \end{equation}
In this case we have $f|_{\chi_{d_\text{eff},e}^d} = f \circ A$.
When we consider that the function we try to approximate is of the form $f \circ A$, we get the following lemma.
\begin{lemma}[Optimal approximation order bi-activation networks: low-dimensional single subspace]\label{lem:ssubspaces}
For function $f(x)$ with $||f||_{W^{n,\infty}([0,1]^d)}\le 1$ where $x$ is restricted to a single canonical subspace $\chi_{d_\text{eff},e}^d$,  there exists $M_{\mathcal{C}_S,bi-\sigma}(a)$ formed as a feed-forward network with $\mathcal C_S=C_6(d,n)\epsilon^{-d_\text{eff}/n}$ elements $a\in\RR^{\mathcal C_S}$ for which 
\[
\min_{a\in\RR^{\mathcal C_S}}\max_{x\in \chi_{d_\text{eff},e}^d} |f(x)-M_{\mathcal C_S,bi-\sigma}(a)(x)|\le \epsilon
\]
where $C_6(d,n)$ may depend on $d$ and $n$, but not on $\epsilon$.
\end{lemma}

    We prove Lemma \ref{lem:ssubspaces}, by showing that $\|f\circ A\|_{W^{n,\infty}[0,1]^{d_\text{eff})}} \leq 1$ and then applying Theorem \ref{thm:main}.
    
\begin{proof}
    For a fixed $d, n \in \mathbb N $, $d_\text{eff} \in \mathbb N$ such that $d_\text{eff}<d $ and $\epsilon \in (0,1)$. We consider without loss of generality a $f$ and $A$ as prescribed, then by upper bounding $\|f \circ A\|_{W^{n,\infty}([0,1]^{d_\text{eff}})}$ by 1, we can apply Theorem \ref{thm:main}. We have for a $\textbf{n}$, with $|{\textbf{n}}|= n$ that
    \begin{align}
        &\mbox{ esssup }_{x\in [0,1]^{d_\text{eff}}} \left|D^{\textbf{n}}(f\circ A)(x)\right|\nonumber\\
        =& \mbox{ esssup }_{x\in [0,1]^{d_\text{eff}}} \left|\partial_{x_1}^{n_1}\partial_{x_2}^{n_2}\dots\partial_{x_{d_\text{eff}}}^{n_{d_\text{eff}}}(f\circ A)(x)\right|\nonumber\\
        =& \mbox{ esssup }_{x\in [0,1]^{d_\text{eff}}}\left|\partial_{x_1}^{n_1}\partial_{x_2}^{n_2}\dots\partial_{x_{d_\text{eff}}}^{n_{d_\text{eff}}-1}\sum_{i=1}^d\partial_{x_i}(f)(Ax)\cdot A_{i{d_\text{eff}}}\right|\nonumber\\
        =& \mbox{ esssup }_{x\in [0,1]^{d_\text{eff}}}\left|\partial_{x_1}^{n_1}\partial_{x_2}^{n_2}\dots\partial_{x_{d_\text{eff}}}^{n_{d_\text{eff}}-1}\partial_{x_{d_\text{eff}}}(f)(Ax)\cdot A_{{d_\text{eff}}{d_\text{eff}}}\right|\nonumber\\
        =& \mbox{ esssup }_{x\in [0,1]^{d_\text{eff}}} \left|\partial_{x_1}^{n_1}\partial_{x_2}^{n_2}\dots\partial_{x_{d_\text{eff}}}^{n_{d_\text{eff}}}(f)(Ax)\right|\label{eq:induc}\\
        =& \mbox{ esssup }_{x\in [0,1]^{d_\text{eff}}} |D^{\textbf{n}}(f)(Ax)|\nonumber\\
        \leq& \mbox{ esssup }_{x\in [0,1]^d} |D^{\textbf{n}}(f)(x)|\leq 1.\nonumber
    \end{align}
    Here in \eqref{eq:induc} we use the argument above $|\mathbf{n}|$ times.  Taking the maximum over $\mathbf n$ gives us that $$\|f \circ A\|_{W^{n,\infty}([0,1]^{d_\text{eff}})} \leq 1.$$
    To finish the proof we  apply Theorem \ref{thm:main}. 
\end{proof}
The reason we introduce the previous lemma is that for all canonical subspaces of dimension $d_\text{eff}$, we can assume without loss of generality that there exists a matrix $A$ of the form of \eqref{eq:Asimple}.

\begin{proof}[Proof of Theorem \ref{thm:subspaces}]
    For any $d, n, d_\text{eff} \in \mathbb N$ such that $d_\text{eff}<d $ and $\epsilon \in (0,1)$, we define  $\Hat f:[0,1]^d \to \R$, as $\Hat{f}|_{\bar{\chi}_{d_\text{eff},e}^d} = \Tilde f $, for all $e \in I$ where $\Tilde f$ is as in Lemma \ref{lem:ssubspaces}, and zero elsewhere.
    Then for $\Hat f$ we have:
    \begin{align}
       \sup\limits_{x \in \bar{\chi}^d_{d_\text{eff}}}|\Hat{f}(x) - f(x)|\leq& \sum_{e\in I}\sup_{x \in\chi_{d_\text{eff},e}^d}|\Hat{f}(x) - f(x)|\leq\frac{2^{d_\text{eff}}d_\text{eff}^n}{n!}\left(\frac{1}{N}\right)^n {d \choose d_\text{eff}}.\label{eq:fhatf}
    \end{align}
    Setting 
    \begin{equation*}\label{eq:N3}
	N = N(\epsilon, d, {d_\text{eff}}, n) =\left\lceil\left(\frac{ n! {d \choose {d_\text{eff}}}}{2^{d_\text{eff}}{d_\text{eff}}^n}\epsilon\right)^{-1/{n}}\right\rceil
    \end{equation*}
    and plugging $N$ in \eqref{eq:fhatf}, we get $\sup\limits_{x \in \chi^d_{d_\text{eff}}}|\Hat{f}(x) - f(x)|\leq \epsilon$.
   Furthermore, by Lemma \ref{lem:ssubspaces} $\Tilde{f}$ can be implemented as a feed-forward network $M_{\mathcal C_S,bi-\sigma}(a)(x)$. Then $\Hat f$ can be formed as the product of these networks, which results in a total feed-forward network $M_{\mathcal C_M,bi-\sigma}(a)(x)$, where $$\mathcal C_M = {d \choose {d_\text{eff}}}{d_\text{eff}}^n(N+1)^{d_\text{eff}}=C_5(d, {d_\text{eff}}, n)\epsilon^{-{d_\text{eff}}/n},$$ which finishes our proof.
\end{proof}
\begin{remark}
    Although the networks in the case of Lemma \ref{lem:ssubspaces} and Theorem \ref{thm:subspaces} have the same $\epsilon$ functional dependence in their number of parameters, the total size of the network $\mathcal{C_S}$ and $\mathcal{C_M}$ will be different, as they also depend in a different way on $d, d_\text{eff}$ and $n$. 
\end{remark}

\section{Conclusions}

We have shown that bi-activation networks, which use both the ReLU and $x^2$ as activation functions, have greater approximation power than ReLU networks. By repurposing a proof of \cite{yarotsky} for ReLU networks, we have derived upper bounds for the number of parameters needed by bi-activation networks to approximate functions in the unit ball of the Sobolev space $W^{n,\infty}([0,1]^d)$ achieving the optimal order $O(\epsilon^{-d/n})$ number of parameters as lower bounded by \cite{devore}.  We also extended our result to analytic functions on $[0,1]^d$ for yet superior $\epsilon$ dependence and to low-dimensional subspaces to overcome the curse of dimensionality.  

Natural extensions of these results are 1) to determine if a feedforward, or another network, with a single nonlinear activation can achieve the optimal order $O(\epsilon^{-d/n})$ number of parameters, and 2) to consider further low-complexity models of $f(x)$ beyond the union of subspaces, see for instance the nested structure considered in \cite{poggio}.

\section*{Acknowledgments}VG and JH would like to thank the EPSRC Centre for Doctoral Training in Mathematics of Random Systems: Analysis, Modelling and Simulation (EP/S023925/1) for its support.  JT is supported by the Hong Kong Innovation and Technology Commission (InnoHK Project CIMDA) and thanks UCLA Department of Mathematics for kindly hosting him during the completion of this manuscript.

\nocite{*}
\bibliographystyle{plain} 
\bibliography{main}

\section*{Appendix}

\begin{proof}[Proof of Lemma \ref{approximationlemma}] \label{proofth1}
Begin by defining a partition of unity $\phi_{\mathbf{m}}$ on the domain $[0,1]^d$:

$$
\sum_{\mathbf{m}} \phi_{\mathbf{m}}(\mathbf{x}) \equiv 1, \quad \mathbf{x} \in[0,1]^d
$$
Here $\mathbf{m} = (m_1,\dots,m_d) \in\{0,1, \ldots, N\}^d$, and $\phi_{\mathbf{m}}$ is defined as
$$
\phi_{\mathbf{m}}(\mathbf{x})=\prod_{k=1}^d \psi\left(3 N\left(x_k-\frac{m_k}{N}\right)\right),
$$
where
$$
\psi(x)= \begin{cases}1, & |x|<1 \\ 0, & 2<|x| \\ 2-|x|, & 1 \leq|x| \leq 2.\end{cases}
$$
Furthermore, note that  $||\psi||_\infty = 1$ and $||\phi_{\textbf{m}}||_{\infty} = 1$ for all ${\textbf{m}}$, and that 
\[
	\mbox{supp } \phi_{\textbf{m}} \subseteq \left\{x:\bigg|x_k-\frac{m_k}{N}\bigg|<\frac{1}{N}\forall k\right\}.
\]

For any ${\textbf{m}}\in \{0,...,N\}^d$, consider the degree$-(n-1)$ Taylor polynomial for the function $f$ at $\textbf{x}={\textbf{m}}/N$:
\[
	P_\textbf{m}(x) = \sum_{{\textbf{n}:|\textbf{n}|<}n} \frac{D^{{\textbf{n}}}f}{\textbf{n}!}\bigg|_{x={\textbf{m}}/N} \left(x-\frac{{\textbf{m}}}{N}\right)^{\textbf{n}},
\]
with the usual conventions ${\textbf{n}}!=\prod_{k=1}^dn_k!$ and $(x-\frac{{\textbf{m}}}{N})^{\textbf{n}}=\prod_{k=1}^d\left(x_k-\frac{m_k}{N}\right)^{n_k}$. Now define an approximation to $f$ by \begin{equation*}\label{taylorexp}
    f_1 = \sum_{{\textbf{m}} \in \{0,...,N\}^d}\phi_{\textbf{m}}P_{\textbf{m}}.
\end{equation*}

We bound the approximation error using the Taylor expansion of $f$:
\begin{align*}
	|f(x)-f_1(x)|&= \left|\sum_{{\textbf{m}}}\phi_{\textbf{m}}(x)\big(f(x)-P_{\textbf{m}}(x)\big)\right| \\
	&\leq \sum_{{\textbf{m}}:|x_k-m_k/N|<1/N\, \forall k} |f(x)-P_{\textbf{m}}(x)| \\
	&\leq 2^d \max_{{\textbf{m}}:|x_k-{m_k}/{N}|<\frac{1}{N}\forall k} |f(x)-P_{\textbf{m}}(x)| \\
	&\leq \frac{2^dd^n}{n!}\left(\frac{1}{N}\right)^n \max_{{\textbf{n}}:|{\textbf{n}}|=n}\mbox{esssup }_{x\in [0,1]^d} |D^{\textbf{n}}f(x)|
\end{align*}

 In the second step, we used the support property for $\phi_{\textbf{m}}$ and the uniform bound on its supremum norm. In the third step, we used the observation that any $x\in [0,1]^d$ belongs to the support of at most $2^d$ functions $\phi_{\textbf{m}}$, in the fourth a standard bound for the Taylor remainder.

Note that, the coefficients of the polynomials $P_{\textbf{m}}$ are uniformly bounded for all $f$:

\[
	P_{\textbf{m}}(x)=\sum_{\textbf{n}:|{\textbf{n}}|<n}a_{{\textbf{m}},\textbf{n}} \left(x-\frac{{\textbf{m}}}{N}\right)^{\textbf{n}}, \quad |a_{{\textbf{m}},\textbf{n}}|\leq 1.
\]
Expanding $f_1$ as follows
\[
	f_1(x) = \sum_{{\textbf{m}}\in \{0,...,N\}^d} \sum_{{\textbf{n}}: |{\textbf{n}}| < n} a_{{\textbf{m}},\textbf{n}}\phi_{\textbf{m}}(x)\left(x-\frac{{\textbf{m}}}{N}\right)^{\textbf{n}}.
\]
completes the proof.

\end{proof}






\end{document}